\documentclass[psfig, a4paper]{article}

\usepackage{amssymb}
\usepackage{graphics}

\usepackage{amsmath}%
\usepackage{amsthm}
\usepackage{amsfonts}%
\usepackage{amssymb}%

\usepackage{xcolor}
\usepackage{pgfplots}
\usepackage{tikz}

\usepackage{dsfont}
\usepackage{graphicx}
\usepackage{epsf}
\usepackage{epsfig}
\usepackage{epic}
\usepackage{lscape}
\usepackage{float}
\usepackage{stmaryrd}
\usepackage{ifsym}
\usepackage{amsmath}
\usepackage{color}
\usepackage{xcolor}
\usepackage{csquotes}
\usepackage[]{hyperref}
\usepackage{authblk}

\newtheorem{theorem}{Theorem}

\newtheorem{lemma}[theorem]{Lemma}

\begin{document}
%\mainmatter  % start of an individual contribution
% first the title is needed
\title{The Maximum Cosine Framework for Deriving Perceptron Based Linear Classifiers}
 \date{}
%\titlerunning{The MCF for deriving linear classifiers}
%\author{Nader H.  Bshouty \inst{1} \and Catherine A.Haddad-Zaknoon \inst{2}}
%\institute{ Department of Computer Science, Technion, Israel. \and  Department of Computer Science, University of Haifa,  Israel.}

\author[1]{Nader H.  Bshouty}
\author[2]{Catherine A. Haddad-Zaknoon}
\affil[1]{Department of Computer Science, Technion, Israel.}
\affil[2]{Department of Computer Science, University of Haifa,  Israel.}

%Macro to draw pretty boxes around things
\newlength{\boxwidth}
\setlength{\boxwidth}{\textwidth}
\addtolength{\boxwidth}{-27 pt}
\newcommand{\boxer}[1]{\begin{center}
                        \fbox{
                                \begin{minipage}{\boxwidth}
                                \vspace*{.1in}
%\maketitle
\label{style}

                                #1
                                \vspace*{.05in}
                                \end{minipage}}
                        \end{center}}																																																																																					

\maketitle
\begin{abstract}
In this work, we introduce a mathematical framework, called the \emph{Maximum Cosine Framework} or \emph{MCF}, for deriving new linear classifiers. The method is based on selecting an appropriate bound on the cosine of the angle between the target function and the algorithm's. To justify its correctness, we use the MCF to show how to regenerate the update rule of Aggressive ROMMA~\cite{ROMMA}. Moreover, we construct a cosine bound from which we build the \emph{Maximum Cosine Perceptron} algorithm or, for short, the \emph{MCP} algorithm. We prove that the MCP shares the same mistake bound like the Perceptron~\cite{PERC}. In addition, we demonstrate the promising performance of the MCP on a real dataset. Our experiments show that, under the restriction of single pass learning, the MCP algorithm outperforms PA~\cite{PA} and Aggressive ROMMA.

\end{abstract}


\section{Introduction}
Large-scale classification problems are characterized by huge datasets, high dimension and sparse examples. Moreover, the feature space is normally unknown to the learner. Therefore, an efficient learning algorithm should comply with two main requirements: (1) single pass over the examples dataset such that, for each example $\mathbf{x}$, the time complexity for processing the example and adapting the algorithm hypothesis is linear in number of the non-zero features in $\mathbf{x}$, and (2) space complexity is linear in the number of the relevant features. Consequently, in real world applications, classification via linear classifiers has gained a lot of attention due to their efficiency in time and memory.

The roots of many papers discussing linear classifiers date back to the Perceptron algorithm~\cite{PERC}. The Perceptron algorithm gained its popularity due to its efficiency in time and space as well as its polynomial mistake bound. The perceptron update rule complies naturally with the space and time requirements. Many algorithms introduced later followed the perceptron update paradigm including ALMA~\cite{ALMA}, NORMA~\cite{NORMA} and PA~\cite{PA}. 

In this work, we introduce a mathematical framework, called the \emph{Maximum Cosine Framework} or \emph{MCF}, for deriving new algorithms that follow the perceptron update scheme. That is, the algorithm observes examples in a sequence of rounds. On round $i$, it constructs its classification hyperplane $\mathbf{w}_i$ incrementally each time the online algorithm makes a prediction mistake or its confidence in the prediction is inadequately low. It updates its classification scheme using an update rule of the form 
$\mathbf{w}_{i+1}=\mathbf{w}_i + y_i\lambda_i\mathbf{a}_i,$
where $\mathbf{a}_i$ is the current observed example and $\lambda_i$ is a parameter. To calculate the parameter $\lambda_i$, we formulate an upper bound on the cosine of the angle between the target hyperplane and the algorithm's one. Then we choose $\lambda_i$ to be the value that optimizes the cosine bound. We argue that the tighter the cosine bound is, the closer we progress towards the target function. To justify the usefulness of the method, we use the MCF to regenerate the update rule of Aggressive ROMMA. In addition, using the MCF, we build a new linear classifier called the \emph{Maximum Cosine Perceptron} or \emph{MCP} for binary classification.  We prove that the MCP shares the same mistake bound like the Perceptron. In addition, we demonstrate the promising performance of the MCP on a real dataset. Our experiments show that, under the restrictions of memory and single pass learning, the MCP algorithm outperforms PA and Aggressive ROMMA. 

This paper is organized as follows. In Sect. \ref{SEC02} we bring a formal definition for the classification problem via linear classifiers. Moreover, we define the cosine bound concept and develop some preliminaries and useful lemmas that will be used along the discussion on algorithm construction under the MCF. In Sect. \ref{SEC03}, we use the MCF to develop another algorithm called New Aggressive ROMMA (NAROMMA). We prove equivalence of NAROMMA to the well known Aggressive ROMMA in terms of the cosine of the angle between the target hyperplane and the algorithm's hypothesis. Section \ref{SEC04} outlines in details the construction of the local cosine bound from which the MCP algorithm generates its update rule. Furthermore, we discuss the mistake bound of the MCP algorithm and prove formally that it has the same mistake bound like the Perceptron, PA and Aggressive ROMMA.  In Sect. \ref{SEC05} we describe the experiments we made to compare the performance of the MCP vs. the well known PA and Aggressive ROMMA.  
\section{The Maximum Cosine Framework}
\label{SEC02}
\subsection{Problem Settings}
In the binary classification setting, a \emph{linear classifier} is an $n$-dimensional  hyperplane that splits the space into two, where the points on the different sides correspond to the positive and negative labels. The target hyperplane is described by an $n$-dimensional vector called the \emph{weight} vector and denoted by $\mathbf{w}\in\mathds{R}^n$. Along the discussion we assume that $\|\mathbf{w}\|_2 = 1$. Our goal is to learn a prediction function, normally denoted by an $n$-dimensional vector $\mathbf{w}_i\in\mathds{R}^n$, from a sequence of training examples $\{(\mathbf{a}_1, y_1),\cdots,(\mathbf{a}_T, y_T)\}$ where $\mathbf{a}_i\in\mathds{R}^n$ and $\|\mathbf{a}_i\|_2\leq R$ for some $R>0$. In addition, $y_i\in\{-1,+1\}$ is the class label assigned to $\mathbf{a}_i$. In the \emph{online learning model}, the learning process proceeds in \emph{trials}. On trial $i$, the learning algorithm observes an example $\mathbf{a}_i$ and \emph{predicts} the classification $\hat{y_i}\in \{-1,+1\}$ such that $\hat{y_i} = sign(\mathbf{w}_i^T\cdot\mathbf{a}_i)$. We say that the algorithm made a \emph{mistake} if $\hat{y_i}\neq y_i$. The magnitude $|\mathbf{w}_i^T\mathbf{a}_i|$ is interpreted as the degree of confidence in the prediction. We refer to the term $y_i(\mathbf{w}_i^T\mathbf{a}_i)$ as the (signed) \emph{margin} attained at round $i$. Let $\gamma >0$, and let $S$ be a set of binary labeled examples. We say that $\mathbf{w}$ seperates $S$ with margin $\gamma$, if for all $\mathbf{a}\in S$, $|\mathbf{w}^T\mathbf{a}| \geq \gamma$. The margin $\gamma$ is unknown to the algorithm. An online learning algorithm is \emph{conservative} if it updates its weight vector $\mathbf{w}_i$ only on mistake, and \emph{non-conservative} if it performs its update when the margin did not achieve a predefined threshold. 

Let $\mathcal{A}$ be some online algorithm for binary classification that is introduced to some examples set. We say that $\mathcal{A}$ follows the \emph{perceptron algorithm scheme}, if it maintains some hypothesis $h_i = sign(\mathbf{w}_i^T\mathbf{x})$ initialized to some value $\mathbf{w}_1$, normally chosen to be $\mathbf{0}$. On each example $\mathbf{a}_i$, the algorithm decides to update its hypothesis according to some predefined condition using the following update rule $\mathbf{w}_{i+1} = \mathbf{w}_i + \lambda_{i}(y_i\mathbf{a}_i) $. 

From now on, we will use the terms \emph{perceptron-like algorithm} and \emph{perceptron algorithm} alternately to point the fact that the algorithm follows the perceptron algorithm scheme.  

Along the discussion in this paper, we restrict our analysis to the case where the algorithm uses $0$ \emph{bias} hypothesis. That is, the target hypothesis is of the form $\mathbf{w}+\theta$ where we assume $\theta=0$. It is well known in the literature that, for the variable bias case, we can get analogues theorems to the ones we prove in this work that are a constant factor worse than the original bounds~\cite{SVMB}.
\subsection{The Cosine Bound}
%As discussed in previous subsections, the update rule of perceptron-like algorithms follows the form $$\mathbf{w}_{i+1}= %\mathbf{w}_i +\lambda_i(y_i\mathbf{a}_i)$$
%where $\mathbf{a}_i$ is the $i$th example received by the algorithm and $\lambda_i$ is some non negative value.
Let $\theta_{i}$ be the angle between the target hypothesis $\mathbf{w}$ and $\mathbf{w}_i$. Recall that $\|\mathbf{w}\|_2=1$. Our target is to choose $\lambda_{i}$ such that
\begin{equation}
\label{alphai}
\alpha_{i+1}\triangleq\cos\theta_{i+1}=\frac{\mathbf{w}^T\mathbf{w}_{i+1}}{\|\mathbf{w}_{i+1}\|_2}
\end{equation}
is maximal. The incentive of this choice is that we want to choose $\mathbf{w}_{i+1}$ to be as close as possible to $\mathbf{w}$. We start by choosing $\mathbf{w}_1=cy_0\mathbf{a}_0$ for some $c>0$. Under the separability assumption, this choice will guarantee that $\alpha_1 > 0$. Since the target hypothesis is unknown to the algorithm, it is obviously clear that we cannot find an accurate value for $\lambda_i$ that maximizes the expression in (\ref{alphai}). Instead, we will formulate a lower bound for $\cos\theta_i$ that we will call the \emph{cosine bound} on which optimality can be achieved. Then, we find the value of $\lambda_i$ that maximizes it. The optimal value of $\lambda_i$ defines the algorithm's update on each trial. It is needless to say that choosing different cosine bounds will derive different update rules, namely different algorithms. To achieve a better classifier, we aim to maximize the value of $\alpha_i$ at each round $i$.
For that purpose, we will start with some lemmas that will assist us develop cosine bounds from which we can derive new perceptron-like algorithms. For clarity, we bring the proofs of the following three lemmas in Appendix \ref{PRFS}.
\begin{lemma}
\label{lemma3}
Let  $\{(\mathbf{a}_0, y_0),\cdots,(\mathbf{a}_T, y_T)\}$  be a sequence of examples where $\mathbf{a}_i\in\mathds{R}^n$ and $y_i\in \{-1,+1\}$ for all $0\leq i \leq T$. Let $\gamma >0$ and  $\mathbf{w}\in\mathds{R}^n$ with $\|\mathbf{w}\|_2=1$, such that for all $0\leq i \leq T$, $|\mathbf{w}^T\mathbf{a}_i|\geq \gamma$ . Let $\mathbf{w}_1=c\mathbf{a}_0$ for any constant $c>0$. Let $\lambda_i>0$ and $\mathbf{w}_{i+1}= \mathbf{w}_i + \lambda_i(y_i\mathbf{a}_i)$ be the update we use after the $i$th example. Let 
\begin{equation}
x_i=\frac{\lambda_i}{\|\mathbf{w}_i\|_2}\geq 0 \enspace .
\label{Xi}
\end{equation}
Then,
\begin{equation}
\alpha_{i+1}\geq\frac{\alpha_i+\gamma x_i}{\sqrt{1+ \| \mathbf{a}_i\|_{2}^{2} x_{i}^{2} +2\frac{y_i(\mathbf{w}_i^T\mathbf{a}_i)}{\|\mathbf{w}_i\|
_2}x_i}} \enspace .
\end{equation}
\end{lemma}
The following two lemmas will assist us in optimizing the value of $\lambda_i$.
\begin{lemma}
\label{lemma4}
Let
\begin{equation}
\Phi(x) = \frac{r+p\cdot x}{\sqrt{s+q\cdot x^2}}\enspace,
\end{equation}
where $s,q>0$. If $r\neq0$ the optimal (maximal or minimal) value of $\Phi(x)$ is in $x^{*} = (ps)/(rq)$ and is equal to 
\begin{equation}
sign(r)\sqrt{\frac{r^2}{s} +\frac{p^2}{q}}\enspace.
\end{equation}
The point $x^*$ is minimal if $r<0$ and maximal if $r>0$. If $r=0$, the function $\Phi(x)$ is monotone increasing if $p>0$ and monotone decreasing if $p<0$.
\end{lemma}
\begin{lemma}
Let
\begin{equation}
\Phi(x)=\frac{r+p\cdot x}{\sqrt{s+q\cdot x^2 +2tq\cdot x}}
\end{equation}
where $s+q\cdot x^2+2tq\cdot x > 0$ for all $x$. Let $\Phi^{*}$ be the maximal value of $\Phi(x)$ over $\mathds{R}^+$, i.e.,
$\Phi^{*}=\max_{x\in \mathds{R}^+}\Phi(x)\enspace.$
Then we have the following cases,
\begin{enumerate}
\item
 If $r-pt=0$ and $p>0$ then, $\Phi^{*}=\Phi(\infty)=p/\sqrt{q}\enspace.$
\item
If $r-pt=0$ and $p<0$ then, $\Phi^{*}=\Phi(0)=r/\sqrt{s}\enspace.$
\item
If $r-pt>0$ and $ps-rtq\geq 0$, let $x^{*}=(ps-rtq)/(rq-ptq)\geq 0$. Then,$ \Phi^{*}=\Phi(x^{*})=\sqrt{(r-pt)^2 /(s-qt^2)+p^2/q}\enspace.$
\item
If $r-pt>0$ and $ps-tqr<0$ then,
$\Phi^{*} = \Phi(0) = r/\sqrt{s}\enspace.$
\item
If $r-pt<0$ and $ps-tqr\geq 0$ then, $\Phi^{*}=\Phi(\infty)=p/\sqrt{q}\enspace. $
\item
If $r-pt<0$ and $ps-tqr<0$ then,\\
$\Phi^{*}=\max{(\Phi(\infty) ,\Phi(0))} = \max{ \left (p/\sqrt{q}, r/\sqrt{s}\right)}\enspace.$
\end{enumerate}
\label{lemma35}
\end{lemma}

\section{New Aggressive ROMMA}
\label{SEC03}
We use the MCF to construct a non-conservative algorithm - the NAROMMA algorithm.  We start by formulating a local cosine bound from which we derive the best choice of $\lambda_i$ at each trial. Along the discussion in this section and Sect. \ref{SEC04}, we assume that the prerequisites of Lemma \ref{lemma3} apply, that is, for a sequence of examples $\{(\mathbf{a}_0, y_0),\cdots,(\mathbf{a}_T, y_T)\}$  there is  $\gamma >0$ and a separating hyperplane
$\mathbf{w}\in\mathds{R}^n$  with $\|\mathbf{w}\|_2 = 1$ such that $|\mathbf{w}^T\mathbf{a}_i| > \gamma$ for all $i$.
\subsection{The Local Cosine Bound for the NAROMMA Algorithm}
In this algorithm, $\mathbf{w}_1 = y_0\mathbf{a}_0$ where $\mathbf{a}_0$ is the first example received by the algorithm, and the update follows the perceptron paradigm, that is $\mathbf{w}_{i+1}= \mathbf{w}_i + \lambda_i(y_i\mathbf{a}_i)$. Let 
\begin{equation}
\gamma_i=\frac{y_i(\mathbf{w}_i^T\mathbf{a}_i)}{\|\mathbf{w}_i\|_2}
\label{gammai}
\end{equation}
which is the projection of $y_i\mathbf{a}_i$ on the direction of $\mathbf{w}_i$.
Let 
\begin{equation}
\delta_i\triangleq \frac{\cos(\theta_i)}{\gamma}=\frac{\mathbf{w}^T\mathbf{w}_i}{\gamma\|\mathbf{w}_i\|_2}=\frac{\alpha_i}{\gamma}\enspace,
\label{deltai}
\end{equation}
where $\alpha_i$ is as defined in~(\ref{alphai}). 
Then, by Lemma \ref{lemma3} we have,
\begin{equation}
\delta_{i+1}\geq \frac{\delta_i+x_i}{\sqrt{1+\|\mathbf{a}_i\|_2^2x_i^2+2\gamma_ix_i}}
\label{deltaiplus1}
\end{equation}
for all $x_i \geq 0 $. Since we cannot find an accurate evaluation for $\delta_i$, we will look alternatively for some lower bound $\ell_i$ for $\delta_i$, which will act as a local cosine bound for the algorithm, such that 
\begin{equation}
\delta_i \geq \ell_i
\end{equation}
for all $i$.
We start by choosing $\ell_1$,
\begin{equation}
\delta_1 = \frac{\cos(\theta_1)}{\gamma}=\frac{y_0(\mathbf{w}^T\mathbf{a}_0)}{\gamma\|\mathbf{a}_0\|_2}\geq \frac{1}{\|\mathbf{a}_0\|_2}=\ell_1\enspace .
\end{equation}
Assuming that $\delta_i \geq \ell_i$ for some $i$, then by (\ref{deltaiplus1}) we get,
\begin{equation}
\delta_{i+1}\geq \frac{\ell_i + x_i}{\sqrt{1+\|\mathbf{a}_i\|_2^2x_i^2+2\gamma_ix_i}}\enspace.
\label{cosineBndMCP}
\end{equation}
Inequality (\ref{cosineBndMCP}) formulates a local cosine bound for the NAROMMA algorithm. Therefore, since it holds for all $x_i\geq 0$, our next step is to find $x_i^*$ that maximizes the right-hand side of (\ref{cosineBndMCP}) and get an optimal lower bound for $\delta_i$. To achieve this, we start by the following lemma,
\begin{lemma}
Let  $\{(\mathbf{a}_0, y_0),\cdots,(\mathbf{a}_T, y_T)\}$  be a sequence of examples where $\mathbf{a}_i\in\mathds{R}^n$ and $y_i\in \{-1,+1\}$ for all $0\leq i \leq T$. Let $\gamma >0$ and  $\mathbf{w}\in\mathds{R}^n$ with $\|\mathbf{w}\|_2=1$, such that for all $0\leq i \leq T$, $|\mathbf{w}^T\mathbf{a}_i|\geq \gamma$. Let,
\begin{equation}
\ell_{i+1}(x_i,\ell_i)\triangleq \frac{\ell_i +x_i}{\sqrt{1+\|\mathbf{a}_i\|_2^2x_i^2+2\gamma_i x_i}}\enspace,
\label{elli}
\end{equation}
and
\begin{equation}
x_i^*(\ell_i)\triangleq\arg\max_{x_i\geq 0}\frac{\ell_i +x_i}{\sqrt{1+\|\mathbf{a}_i\|_2^2x_i^2+2\gamma_i x_i}}
\label{ximax}
\end{equation}
and finally,
\begin{equation}
\ell_{i+1}^*\triangleq\ell_{i+1}(x_i^*(\ell_i^*), \ell_i^*)\enspace,
\label{ellStar}
\end{equation}
where $\ell_1^*=\ell_1=\frac{1}{\|\mathbf{a}_0\|_2}\enspace.$ Then for all $i\geq 1$,
\begin{equation}
\delta_i \geq \ell_i^*\enspace.
\label{deltaell}
\end{equation}
\label{cosBndEllStar}
\end{lemma}
\begin{proof}
We prove by induction. For $i=1$ we have,
$\delta_1 = \frac{\cos(\theta_1)}{\gamma}=\frac{y_0(\mathbf{w}^T\mathbf{a}_0)}{\gamma\|\mathbf{a}_0\|_2}\geq \frac{1}{\|\mathbf{a}_0\|_2}=\ell_1=\ell_1^*\enspace.$
Assume that the lemma holds for $i=j$, i.e. $\delta_j\geq\ell_j^*$, we prove for $i=j+1$. By (\ref{deltaiplus1}) and the induction assumption we get for all $x_j \geq 0$,
$\delta_{j+1}\geq  \frac{\ell_j^*+x_j}{\sqrt{1+\|\mathbf{a}_j\|_2^2x_j^2+2\gamma_jx_j}}\enspace.$
Specifically, the above inequality holds for $x_j=x_j^*(\ell_j^*)$. Using (\ref{elli}) and (\ref{ellStar}) we get,
$\delta_{j+1}\geq  \frac{\ell_j^*+x_j^*(\ell_j^*)}{\sqrt{1+\|\mathbf{a}_j\|_2^2{x_j^*(\ell_j^*)}^2+2\gamma_jx_j^*(\ell_j^*)}}=\ell_{j+1}^*\enspace,$
which proves our lemma.
\qed\end{proof}
The above discussion implicitly assumes that (\ref{elli}) and (\ref{ximax}) are well defined for all $x_i\geq 0$, and that, given some $\ell_i^*$, we can easily determine $x_i^*(\ell_i^*)$ by solving the optimization problem implied by (\ref{ximax}). Therefore, for completeness of the discussion, we first need to show that $1+\|\mathbf{a}_i\|_2^2x_i^2+2\gamma_ix_i > 0$. Second, we need to propose some direct solution from which the optimal $x_i^*$ can be obtained. By calculating the discriminant of  $1+\|\mathbf{a}_i\|_2^2x_i^2+2\gamma_ix_i$, one can conclude that it has a negative value for all $x_i\geq 0$ when  $|\mathbf{w}_i^T\mathbf{a}_i|/(\|\mathbf{w}_i\|_2\|\mathbf{a}_i\|_2) <1$. It is evident that $|\mathbf{w}_i^T\mathbf{a}_i|/\|\mathbf{w}_i\|_2\|\mathbf{a}_i\|_2 = 1$ if and only if $\mathbf{a}_i$ and $\mathbf{w}_i$ are linearly dependent. However, in this case the update will not change the direction of $\mathbf{w}_{i+1}$ regardless of the choice of $\lambda_i$ and hence such examples will be disregarded by the algorithm.  
The following lemma provides a direct way to calculate $\ell_i^*$ and $x_i^*$ for all $i$.
\begin{lemma}
Let  $\{(\mathbf{a}_0, y_0),\cdots,(\mathbf{a}_T, y_T)\}$  be a sequence of examples where $\mathbf{a}_i\in\mathds{R}^n$ and $y_i\in \{-1,+1\}$ for all $0\leq i \leq T$. Let $\gamma >0$ and  $\mathbf{w}\in\mathds{R}^n$ with $\|\mathbf{w}\|_2=1$, such that for all $0\leq i \leq T$, $|\mathbf{w}^T\mathbf{a}_i|\geq \gamma$ . Let $\Phi(x_i)=\frac{\ell_i^*+x_i}{\sqrt{1+\|\mathbf{a}_i\|_2^2x_i^2+2\gamma_ix_i}}$, 
and let $x_i^*=arg\max_{x_i\geq0}\Phi(x_i)$. If $\gamma_i>0$ then, 
\begin{equation}
x_i^* = \left\{ \begin{array}{ll}
\infty&,\ell_i^*\leq\frac{\gamma_i}{\|\mathbf{a}_i\|_2^2}\\
\frac{1-\ell_i^*\gamma_i}{\ell_i^*\|\mathbf{a}_i\|_2^2-\gamma_i}&,\frac{1}{\gamma_i}>\ell_i^*>\frac{\gamma_i}{\|\mathbf{a}_i\|_2^2}\\
0&,\ell_i^* \geq\frac{1}{\gamma_i}
\end{array}\enspace. \right.
\end{equation}\\
If $\gamma_i\leq 0$ then,
\begin{equation}
x_i^*=\frac{1-\ell_i^*\gamma_i}{\ell_i^*\|\mathbf{a}_i\|_2^2-\gamma_i}\enspace.
\end{equation}\\
Also, 
\begin{equation}
{\ell_{i+1}^*}^2 = \left\{ \begin{array}{ll}
\frac{1}{\|\mathbf{a}_i\|_2^2}&, x_i^*=\infty\\
{\ell_i^*}^2&, x_i^*=0\\
\ell_i^2+ \frac{(1-\ell_i^*\gamma_i)^2}{\|\mathbf{a}_i\|_2^2-\gamma_i^2}&, otherwise\\
\end{array} \right.
\end{equation}
where $\ell_1^* = 1/\|\mathbf{a}_0\|_2$.
\label{lemma5_2}
\end{lemma}
\begin{proof}
We use Lemma \ref{lemma35}. Let $r=\ell_i^*$, $p=s=1$, $q=\|\mathbf{a}_i\|_2^2$ and $t=\gamma_i/\|\mathbf{a}_i\|_2^2$. Then 
$r-pt = \ell_i^*-\frac{\gamma_i}{\|\mathbf{a}_i\|_2^2}$, and $ps-rtq=1-\ell_i^*\gamma_i$.
If $\gamma_i\leq 0$ then both $r-pt$ and $ps-rtq$ are positive, and, by Lemma \ref{lemma35} (case 3), we get that the optimal value is in $x_i^* = \frac{(ps-rtq)}{(rq-ptq)}=\frac{1-\ell_i^*\gamma_i}{\ell_i^*\|\mathbf{a}_i\|_2^2-\gamma_i}\enspace.$
Now, suppose $\gamma_i >0$. Since $\gamma_i=y_i(\mathbf{w}_i^T\mathbf{a}_i)/\|\mathbf{w}_i\|_2>0$, then $\gamma_i/\|\mathbf{a}_i\|_2=y_i(\mathbf{w}_i^T\mathbf{a}_i)/\|\mathbf{w}_i\|_2\|\mathbf{a}_i\|_2\leq 1$ and hence, we get $\gamma_i < \|\mathbf{a}_i\|_2$. Since $\gamma_i < \|\mathbf{a}_i\|_2$, we have $\gamma_i/\|\mathbf{a}_i\|_2^2<1/\|\mathbf{a}_i\|_2 < 1/\gamma_i$.  Therefore, we have the following three cases for $\ell_i^*$:
\begin{description}
\item [Case I.]
$\ell_i^*\leq\gamma_i/\|\mathbf{a}_i\|_2^2$. Then, $r-pt = \ell_i^*-\gamma_i/\|\mathbf{a}_i\|_2^2\leq 0$ and, $ps-rtq = 1-\ell_i^*\gamma_i>0$. By cases 1 and 5 in Lemma \ref{lemma35}, we get $x_i^* = \infty$ and $\ell_{i+1}^*= 1/\|\mathbf{a}_i\|_2$.
\item[Case II.]
$1/\gamma_i>\ell_i^*>\gamma_i/\|\mathbf{a}_i\|_2^2$. Then, $r-pt=\ell_i^*-\gamma_i/\|\mathbf{a}_i\|_2^2>0$ and, $ps-rtq=1-\ell_i^*\gamma_i>0$. Hence, by case 3 in Lemma \ref{lemma35} we get $x_i^*=\frac{1-\ell_i^*\gamma_i}{\ell_i^*\|\mathbf{a}_i\|_2^2\gamma_i}$ and
\begin{equation}
\ell_{i+1}^*  =  \ell_{i+1}(x_i^*,\ell_i^*) = \frac{\ell_i^*+x_i^*}{\sqrt{1+\|\mathbf{a}_i\|_2^2{x_i^*}^2+2\gamma_ix_i^*}}\enspace.
\label{liOpt}
\end{equation}
By applying the value of $x_i^*$ into (\ref{liOpt}) the result follows.
\item[Case III.]
$1/\gamma_i\leq\ell_i^*$. Since $1/\gamma_i > \gamma_i/\|\mathbf{a}_i\|_2^2 $ then, $r-pt=\ell_i^*-\gamma_i/\|\mathbf{a}_i\|_2^2 >0$ and $ps-rtq=1-\ell_i^*\gamma_i\leq 0$. Therefore, using case 4 in Lemma \ref{lemma35} we get the result.
\end{description}
\qed \end{proof}
Recall the definition of $x_i$ from (\ref{Xi}).The implication of taking $x_i$ to $\infty$ in the above update is to change the orientation of $\mathbf{w}_i$ to be the same as $\mathbf{a}_i$. That is, in case of $x_i^*=\infty$ we get, $\mathbf{w}_{i+1}= y_i\mathbf{a}_i$ and
$\ell_{i+1}^*=1/\|\mathbf{a}_i\|_2$. We have just proved,
\begin{lemma}
Let  $\{(\mathbf{a}_0, y_0),\cdots,(\mathbf{a}_T, y_T)\}$  be a sequence of examples where $\mathbf{a}_i\in\mathds{R}^n$ and $y_i\in \{-1,+1\}$ for all $0\leq i \leq T$. Let $\gamma >0$ and  $\mathbf{w}\in\mathds{R}^n$ with $\|\mathbf{w}\|_2=1$, such that for all $0\leq i \leq T$, $|\mathbf{w}^T\mathbf{a}_i|\geq \gamma$ .Let $\mathbf{w}_1=y_0\mathbf{a}_0$, $\ell_1^* = 1/\|\mathbf{a}_0\|_2$ and 
\begin{equation}
\mathbf{w}_{i+1} = \left\{ \begin{array}{ll}
\mathbf{w}_i& ,\gamma_i \geq \frac{1}{\ell_i^*},\\
y_i\mathbf{a}_i& ,\frac{1}{\ell_i^*} > \gamma_i \geq \|\mathbf{a}_i\|_2^2\ell_i^*\\
\mathbf{w}_i+\frac{(1-\ell_i^*\gamma_i)\|\mathbf{w}_i\|_2}{\ell_i^*\|\mathbf{a}_i\|_2^2-\gamma_i}(y_i\mathbf{a}_i)& ,\gamma_i < \min\{\ell_i^*\|\mathbf{a}_i\|_2^2, \frac{1}{\ell_i^*}\}.
\end{array} \right.
\label{NAROMMAUpdate}
\end{equation} 
and
\begin{equation}
{\ell_{i+1}^*}^2= \left\{ \begin{array}{ll}
{\ell_i^*}^2 & , \gamma_i \geq \frac{1}{\ell_i^*},\\
\frac{1}{\|\mathbf{a}_i\|_2^2} &, \frac{1}{\ell_i^*} > \gamma_i \geq \|\mathbf{a}_i\|_2^2\ell_i^*\\
{\ell_i^*}^2 + \frac{(\ell_i^*\gamma_i-1)^2}{\|\mathbf{a}_i\|_2^2 - \gamma_i^2  }& ,\gamma_i < \min\{\ell_i^*\|\mathbf{a}_i\|_2^2, \frac{1}{\ell_i^*}\}.
\end{array} \right.
\label{NAROMMAELLI}
\end{equation}
and $\gamma_i$ is as in (\ref{gammai}).
Then after the $i$th update the cosine of the angle between $\mathbf{w}_i$ and $\mathbf{w}$ is at least $$\cos(\theta_i) = \gamma\delta_i \geq\gamma\ell_i^*.$$
\label{lemma5}
\end{lemma}  
Figure \ref{MCPAlg} summarizes the NAROMMA algorithm.\\ 

\begin{figure}
%\vspace{1cm}
\begin{center}
\boxer{ \setlength{\unitlength}{1cm}
\noindent{\textbf{The NAROMMA algorithm}}
\begin{small}
\renewcommand{\labelenumii}{\theenumii}
\renewcommand{\theenumii}{\theenumi.\arabic{enumii}.}
\begin{enumerate}
\item Get $(\mathbf{a}_0,y_0); \ \mathbf{w}_1\gets (y_0\mathbf{a}_0); \ i\gets 1;\ \ell_1^{*}\gets \frac{1}{\|\mathbf{a}_0\|_2}  $.
\item   Get $(\mathbf{x},y)$.
\item $\gamma_i \gets y(\mathbf{w}_i^T\mathbf{x})/\|\mathbf{w}_i\|_2.$
\item  If $|(\mathbf{w}_i^T\mathbf{x})|/(\|\mathbf{w}_i\|_2\|\mathbf{x}\|_2) = 1$, go to 8.
\item  If $\gamma_i \geq 1/\ell_i^*$, then,
	\begin{enumerate}
		\item $\mathbf{a}_i\gets \mathbf{x}; \ y_i\gets y; \ \mathbf{w}_{i+1} \gets \mathbf{w}_i;$ $\ell_{i+1}^* \gets \ell_{i}^*.$
	\end{enumerate}
\item  If $ \frac{1}{\ell_i^*} > \gamma_i \geq \|\mathbf{a}_i\|_2^2\ell_i^*$, then,
	\begin{enumerate}
		\item $\mathbf{a}_i\gets \mathbf{x}; \ y_i\gets y; \ \mathbf{w}_{i+1}\gets y_i\mathbf{a}_i$; $\ell_{i+1}^*\gets 1/\|\mathbf{a}_i\|_2.$
	\end{enumerate}
\item If $\gamma_i < \min\{\ell_i^*\|\mathbf{a}_i\|_2^2, \frac{1}{\ell_i^*}\}$, then,
	\begin{enumerate}
		\item  $\mathbf{a}_i\gets \mathbf{x}; \ y_i\gets y;$ \  
      \item    $\mathbf{w}_{i+1}\gets \mathbf{w}_i+\frac{(1-\ell_i^*\gamma_i)\|\mathbf{w}_i\|_2}{\ell_i^*\|\mathbf{a}_i\|_2^2-\gamma_i}(y_i \mathbf{a}_i); \ell_{i+1}^*\gets \sqrt{{\ell_i^*}^2 + \frac{(\ell_i^*\gamma_i-1)^2}{\|\mathbf{a}_i\|_2^2 - \gamma_i^2}}$.

	\end{enumerate}
\item  $i\gets i+1; $ Go to 2.
\end{enumerate}
\end{small}
}
\end{center}
\caption[ ]{ The NAROMMA algorithm for binary classification.}
\label{MCPAlg}
\end{figure}

\subsection{Equivalence to Aggressive ROMMA}
To prove  that NAROMMA is equivalent to Aggressive ROMMA, it is enough to show that the algorithms' vectors have the same orientation after each update.
\begin{theorem} Let $\mathbf{u}_i$ denote the algorithm hypothesis used by algorithm Aggressive  ROMMA after the $i$th update. Let $\mathbf{u}_1 =(y_0/\|\mathbf{a}_0\|_2^2)\mathbf{a}_0$. Let $\mathbf{v}_i$ be the hypothesis of the NAROMMA algorithm, and let  $\mathbf{v}_1=y_0\mathbf{a}_0$, $\ell_1 = 1/\|\mathbf{a}_0\|_2$.  Then for all $i$,
\begin{enumerate}
\item $\|\mathbf{u}_i\|_2 = \ell_i^*$.
\item There exists some $\tau_i > 0$ such that, $\mathbf{u}_i = \tau_i\mathbf{v}_i$.
\end{enumerate}
\end{theorem}
\begin{proof}
We prove by induction on $i$.
For $i=1$, the theorem trivially holds. \\
Assume that the theorem holds for $i=t$, that is,
\begin{equation}
\label{indAsum1}
\|\mathbf{u}_{t}\|_2 = \ell_{t}^*
\end{equation}
and,
\begin{equation}
\label{indAsum2}
\mathbf{u}_{t} = \tau_{t}\mathbf{v}_{t}
\end{equation}
for some $\tau_{t} > 0$. To ease our discussion, we rewrite Aggressive ROMMA update as proposed in~\cite{ROMMA} as detailed in Table~(\ref{TABLE00}),
\begin{table}%[here]
\centering
\caption{Aggressive ROMMA update summary}
\begin{tabular}{c||c|c}
\hline\noalign{\smallskip}
Type & Condition & $\mathbf{u}_{i+1}$\\
\hline
\hline
I&$y_i(\mathbf{u}_i^T\mathbf{a}_i)\geq 1$& $\mathbf{u}_i$\\
\hline
II&$1 > y_i(\mathbf{u}_i^T\mathbf{a}_i)\geq \|\mathbf{a}_i\|_2^2\|\mathbf{u}_i\|_2^2$& $\frac{y_i}{\|\mathbf{a}_i\|_2^2}\mathbf{a}_i$\\
\hline
III&$y_i(\mathbf{u}_i^T\mathbf{a}_i)< \min\{ \|\mathbf{a}_i\|_2^2\|\mathbf{u}_i\|_2^2 , 1\}$ &$ c_i\mathbf{u}_i + d_i\mathbf{a}_i$\\
\hline
\end{tabular}
\label{TABLE00}
\end{table}
\\
where 
\begin{equation}
c_i = \frac{\|\mathbf{a}_i\|_2^2\|\mathbf{u}_i\|_2^2-y_i(\mathbf{u}_i^T\mathbf{a}_i)}{\|\mathbf{a}_i\|_2^2\|\mathbf{u}_i\|_2^2-(\mathbf{u}_i^T\mathbf{a}_i)^2}
\label{ci}
\end{equation}
and
\begin{equation}
d_i = \frac{\|\mathbf{u}_i\|_2^2(y_i-(\mathbf{u}_i^T\mathbf{a}_i))}{\|\mathbf{a}_i\|_2^2\|\mathbf{u}_i\|_2^2-(\mathbf{u}_i^T\mathbf{a}_i)^2}\enspace .
\label{di}
\end{equation} 
In the same fashion, Table~(\ref{TABLE01}) summarizes the NAROMMA update rule.
\begin{table}%[here]
\centering
\caption{NAROMMA update summary}
\begin{tabular}{c||c|c|c}
\hline\noalign{\smallskip}
Type & Condition &$\mathbf{v}_{i+1}$&   $\ell_{i+1}^*$ \\
\hline
\hline
 I&$\gamma_i \geq 1/\ell_i^*;$ &$\mathbf{v}_i;$   & $ \ell_i^{*}$;\\
\hline
II&$1/\ell_i^*>\gamma_i \geq \|\mathbf{a}_i\|_2^2\ell_i^*;$   &$y_i\mathbf{a}_i;$   & $\frac{1}{\|\mathbf{a}_i\|_2};$\\
\hline
III&$\gamma_i < \min \{\ell_i^*\|\mathbf{a}_i\|_2^2, 1/\ell_i^*\};$   &$ \mathbf{v}_i + \frac{\|\mathbf{v}_i\|_2(1-\ell_i^*\gamma_i)}{\ell_i^*\|\mathbf{a}_i\|_2^2-\gamma_i}y_i\mathbf{a}_i;$   & $ \ell_i^{*2}+ \frac{(\ell_i^*\gamma_i-1)^2}{\|\mathbf{a}_i\|_2^2-\gamma_i^2}$;\\
\hline
\end{tabular}
\label{TABLE01}
\end{table}
\\
Let $i = t+1$.  Let $\mathbf{\hat{u}}_t$ and $\mathbf{\hat{v}}_t$ denote the unit vectors in the directions of $\mathbf{u}_t$ and $\mathbf{v}_t$, respectively. According to (\ref{indAsum1}) and (\ref{indAsum2}), we get that $$\mathbf{u}_t = \|\mathbf{u}_t\|_2\mathbf{\hat{u}}_t = \|\mathbf{u}_t\|_2\mathbf{\hat{v}}_t=\ell_t^*\mathbf{\hat{v}}_t.$$  Hence, since $\ell_t^* >0$ and using (\ref{gammai}) we get,
\begin{equation}
y_t(\mathbf{u}_t^T\mathbf{a}_t)= y_t\ell_t^*(\mathbf{\hat{v}}_t^T\mathbf{a}_t)=y_t\ell_t^*\frac{1}{\|\mathbf{v}_t\|_2}(\mathbf{v}_t^T\mathbf{a}_t)=\ell_t^*\gamma_t\enspace .
\label{indRes}
\end{equation}
We divide our discussion into three cases according to the update types.
\begin{description}
\item[Case I] Assume that Aggressive ROMMA does not perform any update on the example $\mathbf{a}_t$ (update of type I). Then, according to Table (\ref{TABLE00}), $y_t(\mathbf{u}_t^T\mathbf{a}_t) \geq 1$. Using (\ref{indRes}) and the fact that $\ell_t^* > 0$ we can conclude,
\begin{equation}
y_t(\mathbf{u}_t^T\mathbf{a}_t) \geq 1 \Leftrightarrow \gamma_t \geq 1/\ell_t^*.
\label{equiC1}
\end{equation}
Equation~(\ref{equiC1}) implies that NAROMMA  performs an update of type \emph{I} if and only if Aggressive ROMMA performs an update of type \emph{I}. Hence, the theorem holds for this case.
  \item[Case II] If Aggressive ROMMA performs an update of type \emph{II}, then by Table (\ref{TABLE00}) we get that $1 > y_t(\mathbf{u}_t^T\mathbf{a}_t) \geq \|\mathbf{a}_t\|_2^2\|\mathbf{u}_t\|_2^2$. Therefore, since $\ell_t^* > 0$ and by (\ref{indAsum1}), (\ref{indAsum2}) and (\ref{indRes}) we get,
\begin{equation}
1 > y_t(\mathbf{u}_t^T\mathbf{a}_t) \geq \|\mathbf{a}_t\|_2^2\|\mathbf{u}_t\|_2^2 \Leftrightarrow 1 > \ell_t^*\gamma_t \geq \|\mathbf{a}_t\|_2^2\ell_t^{*2}\Leftrightarrow 1/\ell_t^* > \gamma_t \geq \|\mathbf{a}_t\|_2^2\ell_t^{*} \enspace.
\end{equation}

That is, as in the first case, Aggressive ROMMA and NAROMMA will make their type \emph{II} update simultaneously. Hence,  we get that $\mathbf{v}_{t+1} =y_t\mathbf{a}_t$, $\ell_{t+1}^* = 1/\|\mathbf{a}_t\|_2$ and $\mathbf{u}_{t+1}=(y_t/\|\mathbf{a}_t\|_2^2)\mathbf{a}_t$. Therefore, the theorem trivially follows.
  \item[Case III] From the discussion in the previous two cases, it follows that Aggressive ROMMA makes an update of type \emph{III} if and only if NAROMMA performs a type \emph{III} update.
By Table~(\ref{TABLE00}) we get for this case, 
\begin{equation}
\|\mathbf{u}_{t+1}\|_2^2 = c_t^2\|\mathbf{u}_t\|_2^2 + d_t^2\|\mathbf{a}_t\|_2^2 +2c_td_t(\mathbf{u}_t^T\mathbf{a}_t)\enspace .
\end{equation} 
Since $y_t \in \{-1,1\}$ and using~(\ref{indAsum1}), (\ref{indAsum2}),(\ref{ci}),(\ref{di}) and (\ref{indRes}) we get that,
\begin{align}
\|\mathbf{u}_{t+1}\|_2^2 %&=\frac{(\|\mathbf{a}_t\|_2^2\ell_t^{*2}-\ell_t^{*}\gamma_t)^2\ell_t^{*2} +2\ell_t^{*2}(\|\mathbf{a}_t\|_2^2\ell_t^{*2}-\ell_t^{*}\gamma_t)(1-\ell_t^{*}\gamma_t)\ell_t^*\gamma_t+\ell_t^{*4}(1-\ell_t^{*}\gamma_t)^2\|\mathbf{a}_t\|_2^2}{\ell_t^{*4}\left(\|\mathbf{a}_t\|_2^2 - \gamma_t^2 \right)^2}\nonumber\\
%&= \frac{\ell_t^{*2}(\|\mathbf{a}_t\|_2^2 - \gamma_t^2) - \ell_t^{*2}(\|\mathbf{a}_t\|_t^2-\gamma_t^2)+ \ell_t^{*2}\|\mathbf{a}_t\|_2^2 - 2\ell_t^*\gamma_t+1}{\|\mathbf{a}_t\|_2^2-\gamma_t^2}\nonumber \\
&= \ell_t^{*2}+\frac{(\ell_t^*\gamma_t-1)^2}{\|\mathbf{a}_t\|_2^2-\gamma_t^2}=\ell_{t+1}^{*2}\label{normliprf}.
\end{align}
Moreover, by choosing $\tau_{t+1} =(c_t\|\mathbf{u}_t\|_2)/\|\mathbf{v}_t\|_2$, and using the induction assumption we can easily conclude that $\tau_{t+1}\mathbf{v}_{t+1} = \mathbf{u}_{t+1}$. Hence, the theorem holds for this case too.
\end{description}
\qed\end{proof}
%%%%%%%%%%%%%%%%%%%%%%%%%%%%%%%%%%%%%%%%%%%%%%%%%%%%%%%%%%
\section{The Maximum Cosine Perceptron Algorithm}
\label{SEC04}
We will use the MCF to generate the update rule of the MCP algorithm. The MCP algorithm is non-conservative, i.e. it updates its hypothesis on margin violation. We start our discussion by formulating a conservative algorithm, the \emph{Conservative Maximum Cosine Perceptron} (CMCP) algorithm . By Lemma \ref{lemma3} we have,

\begin{equation}
\alpha_{i+1}\geq \frac{\alpha_i+\gamma x_i}{\sqrt{1+\|\mathbf{a}_i\|_2^2x_i^2+2\frac{y_i(\mathbf{w}_i^T\mathbf{a}_i)}{\|\mathbf{w}_i\|_2}x_i}} \enspace,
\label{mcpAlphai}
\end{equation}
where $x_i=\lambda_i/\|\mathbf{w}_i\|_2$. The CMCP is a conservative algorithm hence, it makes an update when $y_i(\mathbf{w}_i^T\mathbf{a}_i)\leq 0$. Then, from (\ref{mcpAlphai}) we can write,
\begin{equation}
\alpha_{i+1}\geq \frac{\alpha_i+\gamma x_i}{\sqrt{1+\|\mathbf{a}_i\|_2^2x_i^2}} \enspace.
\end{equation}
Assuming that $\alpha_i \geq \gamma\ell_i$ for some $\ell_i$ we have,
\begin{equation}
\alpha_{i+1}\geq\gamma\frac{\ell_i+x_i}{\sqrt{1+\|\mathbf{a}_i\|_2^2x_i^2}}\enspace.
\label{cosBndMain}
\end{equation} 
Inequality (\ref{cosBndMain}) formulates a local cosine bound for the CMCP algorithm. It holds for all $x_i\geq 0$, hence, by maximizing its right hand side we get an optimal lower bound for $\alpha_i$.  
By Lemma \ref{lemma4}, optimality is obtained in $x_i = 1/(\ell_i\|\mathbf{a}_i\|_2^2)$ and then,
\begin{equation}
\alpha_{i+1}\geq \gamma\sqrt{\ell_i^2 + \frac{1}{\|\mathbf{a}_i\|_2^2}}\enspace.
\end{equation}
Let $\mathbf{w}_1 = y_0\mathbf{a}_0$ and $\ell_1 = 1/\|\mathbf{a}_0\|_2$, we get that,
$\alpha_1 = \frac{\mathbf{w}^T\mathbf{w}_1}{\|\mathbf{w}_1\|_2}=\frac{|\mathbf{w}^T\mathbf{a}_0|}{\|\mathbf{a}_0\|_2}\geq \gamma\frac{1}{\|\mathbf{a}_0\|_2} = \gamma\ell_1\enspace.$  And hence, by choosing
\begin{equation}
\ell_{i+1}^2 = \ell_i^2 + \frac{1}{\|\mathbf{a}_i\|_2^2} = \sum_{j=0}^{i}\frac{1}{\|\mathbf{a}_j\|_2^2}
\end{equation}
we can conclude,
\begin{equation}
\alpha_{i}\geq \gamma\sqrt{\sum_{j=0}^{i-1}\frac{1}{\|\mathbf{a}_j\|_2^2}}=\gamma\ell_i\enspace.
\end{equation}
We have just proved,
\begin{lemma}
Let  $\{(\mathbf{a}_0, y_0),\cdots,(\mathbf{a}_T, y_T)\}$  be a sequence of examples where $\mathbf{a}_i\in\mathds{R}^n$ and $y_i\in \{-1,+1\}$ for all $0\leq i \leq T$. Let $\gamma >0$ and  $\mathbf{w}\in\mathds{R}^n$ with $\|\mathbf{w}\|_2=1$, such that for all $0\leq i \leq T$, $|\mathbf{w}^T\mathbf{a}_i|\geq \gamma$. Let $\mathbf{w}_1=y_0\mathbf{a}_0$, $\ell_1=1/\|\mathbf{a}_0\|_2$ and
\begin{equation}
\mathbf{w}_{i+1} = \left\{ \begin{array}{ll}
\mathbf{w}_i+ \frac{\|\mathbf{w}_i\|_2}{\ell_i\|\mathbf{a}_i\|_2^2}(y_i\mathbf{a}_i)& ,y_i(\mathbf{w}_i^T\mathbf{a}_i)\leq 0\\
\mathbf{w}_i& , otherwise
\end{array} \right.
\label{cmcpUpdate}
\end{equation}
be the update we use in the $i$th example, and
\begin{equation}
{\ell_{i+1}^2 = \ell_i^2 + \frac{\mu_i}{\|\mathbf{a}_i\|_2^2}}\enspace,
\label{LiUpdate}
\end{equation}
\begin{equation}
\mu_{i} = \left\{ \begin{array}{ll}
1&,y_i(\mathbf{w}_i^T\mathbf{a}_i)\leq 0\\
0&, otherwise.
\end{array} \right.
\label{etaupUpdate}
\end{equation}
%$$\ell_i = \sum_{j=0}^{i-1}\frac{1}{\|\mathbf{a}_j\|_2^2}.$$
Then, after the $i$th example, $\alpha_i$ is at least $\gamma\ell_i$.
\label{lemma7}
\end{lemma}

To convert the CMCP algorithm to a non-conservative one, we use the same update rule not only on mistakes but also when the example is close to the algorithm's hyperplane, according to the following update rule,
\begin{equation}
\mathbf{w}_{i+1} = \left\{ \begin{array}{ll}
\mathbf{w}_i+ \frac{\|\mathbf{w}_i\|_2}{\ell_i\|\mathbf{a}_i\|_2^2}(y_i\mathbf{a}_i)& ,y_i(\mathbf{w}_i^T\mathbf{a}_i)\leq \frac{\|\mathbf{w}_i\|_2}{2\ell_i}\\
\mathbf{w}_i& , otherwise.
\end{array} \right.
\label{MCPUpdateRule}
\end{equation}
Moreover, the value of $\ell_i$ is updated as follows,
\begin{equation}
\ell_{i+1}^2 = \left\{ \begin{array}{ll}
\ell_i^2 + \frac{1-2\eta_i}{\|\mathbf{a}_i\|_2^2}& , y_i(\mathbf{w}_i^T\mathbf{a}_i) \leq\frac{\|\mathbf{w}_i\|_2}{2\ell_i} \\
\ell_i^2 &, otherwise
%\ell_i^2 + \frac{\eta_i}{\|\mathbf{a}_i\|_2^2}& , y_i(\mathbf{w}_i^T\mathbf{a}_i) \leq\frac{\|\mathbf{w}_i\|_2}{2\ell_i} \\
%\ell_i^2 &, otherwise

\end{array}\right .
\label{mcpLiUpdate}
\end{equation}
where for $ y_i(\mathbf{w}_i^T\mathbf{a}_i) \leq\frac{\|\mathbf{w}_i\|_2}{2\ell_i}$,
\begin{equation}
\eta_i = \left\{ \begin{array}{ll}
0& ,y_i(\mathbf{w}_i^T\mathbf{a}_i)\leq 0\\
\frac{y_i(\mathbf{w}_i^T\mathbf{a}_i)\ell_i}{\|\mathbf{w}_i\|_2}& , 0 < y_i(\mathbf{w}_i^T\mathbf{a}_i) \leq \frac{\|\mathbf{w}_i\|_2}{2\ell_i} .
%1& ,y_i(\mathbf{w}_i^T\mathbf{a}_i)\leq 0\\
%1-2\frac{y_i(\mathbf{w}_i^T\mathbf{a}_i)\ell_i}{\|\mathbf{w}_i\|_2}& , 0 < y_i(\mathbf{w}_i^T\mathbf{a}_i) \leq \frac{\|\mathbf{w}_i\|_2}{2\ell_i} .
\end{array} \right.
\label{EtaUpdateRule}
\end{equation}
Notice that when $\mathbf{a}_i$ is a counterexample, it contributes $1/\|\mathbf{a}_i\|_2^2$ to $\ell_i^2$. Otherwise, it contributes $1-2\eta_i/\|\mathbf{a}_i\|_2^2$, which is always positive because of the update rule condition in (\ref{MCPUpdateRule}). Figure \ref{mcpAlgNew} summarizes the algorithm. 

\begin{figure}
\begin{center}
\boxer{ \setlength{\unitlength}{1cm}
\noindent{\textbf{The MCP algorithm}}
\begin{small}
\renewcommand{\labelenumii}{\theenumii}
\renewcommand{\theenumii}{\theenumi.\arabic{enumii}.}
\begin{enumerate}
\item  Get $(\mathbf{a}_0,y_0); \ \mathbf{w}_1\gets (y_0\mathbf{a}_0); \ i\gets 1;\ \ell_1\gets \frac{1}{\|\mathbf{a}_0\|_2} $ .
\item   Get $(\mathbf{x},y)$.
\item If $y(\mathbf{w}_i^T\mathbf{x})\leq \frac{\|\mathbf{w}_i\|_2}{2\ell_i}$, then,
	\begin{enumerate}
		\item $\mathbf{a}_i\gets \mathbf{x}; \ y_i\gets y; \ \mathbf{w}_{i+1}\gets \mathbf{w}_i + \frac{\|\mathbf{w}_i\|_2}{\ell_i\|\mathbf{a}_i\|_2^2}(y_i\mathbf{a}_i).$
		\item If  $y_i(\mathbf{w}_i^T\mathbf{a}_i) \leq  0$, Then, $\eta_i \gets 0$.
		\item Else, $\eta_i \gets \frac{y_i(\mathbf{w}_i^T\mathbf{a}_i)\ell_i}{\|\mathbf{w}_i\|_2}$.
		\item $\ell_{i+1} = \sqrt{\ell_i^2 + \frac{1-2\eta_i}{\|\mathbf{a}_i\|_2^2}}$. 
	\end{enumerate}
\item Else, $\mathbf{w}_{i+1} \gets \mathbf{w}_i;$ $\ell_{i+1} \gets \ell_{i}.$
\item $i\gets i+1;$ Go to 2.
\end{enumerate}
\end{small}
}
\end{center}
\caption{\sl The MCP algorithm for binary classification.}
\label{mcpAlgNew}
\end{figure}

\begin{lemma}
Let  $\{(\mathbf{a}_0, y_0),\cdots,(\mathbf{a}_T, y_T)\}$  be a sequence of examples where $\mathbf{a}_i\in\mathds{R}^n$ and $y_i\in \{-1,+1\}$ for all $0\leq i \leq T$. Let $\gamma >0$ and  $\mathbf{w}\in\mathds{R}^n$ with $\|\mathbf{w}\|_2=1$, such that for all $0\leq i \leq T$, $|\mathbf{w}^T\mathbf{a}_i|\geq \gamma$. Let $\mathbf{w}_1=y_0\mathbf{a}_0$ and $\ell_1=1/\|\mathbf{a}_0\|_2$. Let $\mathbf{w}_{i+1}$ and $\ell_{i+1}$ be updated  as defined in  (\ref{MCPUpdateRule}) and (\ref{mcpLiUpdate}) respectively, after each example $\mathbf{a}_i$.
Then after the $i$th example, $\alpha_i$  is at least $\gamma\ell_i$.
\label{lemma8}
\end{lemma}
\begin{proof}
We prove by induction on $i$, the index of the current example.
For $i=1$, 
$\alpha_1 = \frac{\mathbf{w}^T\mathbf{w}_1}{\|\mathbf{w}_1\|_2} = \frac{|\mathbf{w}^T\mathbf{a}_0|}{\|\mathbf{a}_0\|_2}\geq \gamma/\|\mathbf{a}_0\|_2 =\gamma\ell_1 \enspace.$
We assume the lemma holds for $i=k$, that is, $\alpha_k \geq \gamma \ell_k$, and prove for $i=k+1$. We consider three cases,
\begin{description}
\item[Case I] if $y_k(\mathbf{w}_k^T\mathbf{a}_k)> \|\mathbf{w}_k\|_2/(2\ell_k)$, then, by (\ref{MCPUpdateRule}) and (\ref{mcpLiUpdate}) no update happens for $\mathbf{w}_k$ or $\ell_k$  and hence, the lemma is true for this case.
\item[Case II] If $y_k(\mathbf{w}_k^T\mathbf{a}_k)\leq 0$, then, by Lemma \ref{lemma3}, the induction assumption and (\ref{MCPUpdateRule}), (\ref{mcpLiUpdate}), (\ref{EtaUpdateRule}) and using $x_k =1/(\ell_k\|\mathbf{a}_k\|_2^2)$ we can write,
\begin{align}
\alpha_{k+1}&\geq\frac{\alpha_k + \gamma x_k}{\sqrt{1 + \|\mathbf{a}_k\|_2^2x_k^2+ 2\frac{y_k(\mathbf{w}_k^T\mathbf{a}_k)}{\|\mathbf{w}_k\|_2}x_k}} \geq \frac{ \gamma\left (\ell_k +\frac{1}{\ell_k\|\mathbf{a}_k\|_2^2} \right )}{\sqrt{1 + \frac{1}{\ell_k^2\|\mathbf{a}_k\|_2^2}}}\geq \nonumber\\ 
 &\geq\gamma \sqrt{\ell_k^2 + \frac{1}{\|\mathbf{a}_k\|_2^2}} \geq \gamma \ell_{k+1}\enspace. \nonumber
\end{align}
\item[Case III] If $0<y_k(\mathbf{w}_k^T\mathbf{a}_k) \leq \|\mathbf{w}_k\|_2/(2\ell_k)$, then, by Lemma \ref{lemma3}, the induction assumption, (\ref{MCPUpdateRule}), (\ref{mcpLiUpdate}) and (\ref{EtaUpdateRule}) we can write,
\begin{align}
\alpha_{k+1}&\geq\frac{\alpha_k + \gamma x_k}{\sqrt{1 + \|\mathbf{a}_k\|_2^2x_k^2+ 2\frac{y_k(\mathbf{w}_k^T\mathbf{a}_k)}{\|\mathbf{w}_k\|_2}x_k}} \geq \nonumber \\ &\geq \frac{\gamma\ell_k\left (1+\frac{1}{\ell_k^2\|\mathbf{a}_k\|_2^2}\right )}{\sqrt{1+\frac{1}{\ell_k^2\|\mathbf{a}_k\|_2^2}  +\frac{2}{\ell_k^2\|\mathbf{a}_k\|_2^2}\left (  \frac{y_k (\mathbf{w}_k^T\mathbf{a}_k)\ell_k}{\|\mathbf{w}_k\|_2} \right )}}\geq  \nonumber\\
&\geq \gamma\ell_{k} \frac{1+\frac{1}{\ell_k^2\|\mathbf{a}_k\|_2^2}}{\sqrt{1+ \frac{1+2\eta_k}{\ell_k^2\|\mathbf{a}_k\|_2^2}}} \geq \gamma\ell_{k}\sqrt{1+\frac{1-2\eta_k}{\ell_k^2\|\mathbf{a}_k\|_2^2}} = \gamma\ell_{k+1}\enspace. \nonumber 
\end{align}
\end{description}
\qed\end{proof}
Lemma \ref{lemma8} motivates the choice of the update condition of the MCP algorithm.  Let $\mathbf{a}_i$ be the current example examined by the algorithm, and let us assume without loss of generality that $y_i=+1$. When  $y_i(\mathbf{w}_i^T\mathbf{a}_i) > \|\mathbf{w}_i\|_2/(2\ell_i)$, by (\ref{MCPUpdateRule}) no update  happens. Let $\theta_i$ be the angle between the target hyperplane and the algorithm hyperplane. Let $\gamma_i$ be as defined in (\ref{gammai}). By Lemma \ref{lemma8} we get that,
$\gamma_i  >  \frac{1}{\ell_i} >\frac{\gamma}{2\alpha_i}= \frac{1}{2}\frac{\gamma}{\cos\theta_i}>\frac{\gamma}{2}$
which implies that the example has the right label and is at distance of at least $\gamma/2$ from $\mathbf{w}_i$ and therefore no update occurs. 
\begin{theorem}
Let $S=(\mathbf{a}_0,y_0),\cdots,(\mathbf{a}_k,y_k)$ be a sequence of examples where $\mathbf{a}_i\in{\mathds{R}}^n$, $y_i\in\{-1,+1\}$ and $\|\mathbf{a}_i\|_2\leq R$ for all $0\leq i\leq k$. Let $\mathbf{w}$ be some separating hyperplane for $S$, that is, there exists some $\gamma>0$ such that for all $i$ such that $0\leq i\leq k$, $|\mathbf{w}^T\mathbf{a}_i|\geq \gamma$. And let $\|\mathbf{w}\|_2=1$. Let $t$ be the number of mistakes the MCP algorithm makes on $S$, then, $t\leq \left (R/\gamma\right )^2$. 
\end{theorem}
\begin{proof}
Let $M\subseteq S$ denote the set of examples on which the MCP algorithm made a mistake. Similarly, let $N$ be the set of examples on which the MCP algorithm made an update. Clearly, $M\subseteq N$. By  (\ref{mcpLiUpdate}) and (\ref{EtaUpdateRule}) we get that $1-2\eta_i >0$ for all $i$. Hence, we can conclude,
\begin{equation}
\ell_k^2 = \sum_{i: \mathbf{a}_i\in N}\frac{1-2\eta_i}{\|\mathbf{a}_i\|_2^2} \geq \sum_{i: \mathbf{a}_i\in M}\frac{1}{\|\mathbf{a}_i\|_2^2} \geq \sum_{i: \mathbf{a}\in M}\frac{1}{R^2} = \frac{t}{R^2}.
\label{res1}
\end{equation}
From Lemma (\ref{lemma8}) we get that,
\begin{equation}
1\geq\alpha_k\geq\gamma\ell_k\geq\gamma\sqrt{\frac{t}{R^2}}\enspace.
\label{res2}
\end{equation}
By combining (\ref{res1}) and (\ref{res2}) we get the result.
\qed\end{proof}
\section{Experiments}
\label{SEC05}
In this section we present experimental results that demonstrate the performance of the MCP algorithm vs. the well known algorithms PA and Aggressive ROMMA on the MNIST OCR database\footnote{See http://yann.lecun.com/exdb/mnist/ for information on obtaining this dataset.}. Every example in the MNIST database has two parts, the first is $28\times 28$ matrix which represents the image of the corresponding digit. Each entry in the matrix takes values from $\{0,\cdots,255\}$. The second part is a label taking values from $\{0,\cdots,9\}$. The dataset consists of 60000 training examples and 10000 test examples. In the experiments we trained the algorithms for single preselected label $l$. When training on this, we replaced each labeled instance $(\mathbf{a}_i, y_i)$ by the binary-labeled instance $(\mathbf{a}_i, y_i^*)$, where $y_i^*=1$ if $y_i = l$ and  $y_i^*=-1$ otherwise. We have divided the training set to 60 buckets of examples each containing 1000 examples. For each label,  we first chose a random permutation of the examples buckets, then we trained the algorithm via single pass over the training dataset according to the  selected permutation. Then, we tested it on the test dataset. We repeated that for 20 random permutations for each label. At the end of the process, to calculate the mistake rates of each classifier, we took the average of the mistakes over the 20 rounds. Figure~\ref{EXEC5} summarizes the number of mistakes made by the three algorithms for all the ten labels on the test data. Actually it shows that MCP practically performs better than the other two algorithms under the restrictions of single dataset pass and hypothesis size that is linear in the number of the relevant features.

\begin{figure}
\begin{center}
\boxer{ \setlength{\unitlength}{1cm}
\centering \includegraphics[scale=0.55]{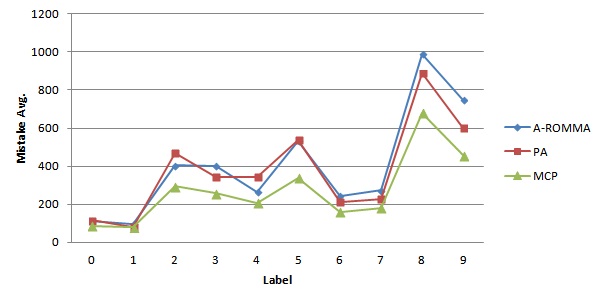}
}
\end{center}
\caption{Single label classifier mistake rates of MCP, PA, A-ROMMA on MNIST dataset. }
\label{EXEC5}
\end{figure}

\appendix
\section{ Proofs of Sect. \ref{SEC02} }
\label{PRFS}
\begin{proof}
(Lemma \ref{lemma3}) .
\begin{align}
\alpha_{i+1} &= \frac{\mathbf{w}^T\mathbf{w}_{i+1}}{\|\mathbf{w}_{i+1}\|_2} = \frac{\|\mathbf{w}_i\|_2\alpha_i+\|\mathbf{w}_i\|_2x_i|\mathbf{w}^T\mathbf{a}_i|}{\sqrt{\|\mathbf{w}_i\|^2_2+ \lambda_i^2\|\mathbf{a}_i\|_2^2+2\lambda_iy_i(\mathbf{w}_i^T\mathbf{a}_i)}} \geq  \nonumber \\
&\geq\frac{\alpha_i +\gamma x_i}{\sqrt{1+x_i^2\|\mathbf{a}_i\|_2^2+2x_i\frac{y_i(\mathbf{w}_i^T\mathbf{a}_i)}{\|\mathbf{w}_i\|_2}}}\enspace \nonumber .
\end{align}
This implies the result.
\qed\end{proof}

\begin{proof}
(Lemma \ref{lemma4}). From solving ${\partial{\Phi(x)}}/{\partial{x}}=0$ and checking the sign of  $\frac{\partial{\Phi(x)}}{\partial{x}}$ we get the result.
\qed\end{proof}

\begin{proof}
( Lemma \ref{lemma35}). We write $\Phi(x)$ in the following manner
\begin{equation}
\Phi(x)=\frac{(r-pt)+p(x+t)}{\sqrt{(s-qt^2)+q(x+t)^2}}\equiv \frac{r'+p'x'}{\sqrt{s'+q'{x'}^{2}}}\enspace,
\end{equation}
where $r^{'} = r-pt$, $p^{'}=p$, $x^{'} = x+t$, $s^{'}=s-qt^2$ and $q^{'}=q$. Since $s+q\cdot x^2+2tq\cdot x> 0$ for all $x$, we have $q>0$ and $\Delta = 4t^2q^2-4sq<0$ which implies $q'>0$ and $s'>0$. Now by Lemma \ref{lemma4}, if $r^{'}=r-pt\neq 0$, the optimal value of $\Phi(x)$ is in 
\begin{equation}
\label{x0}
x_0={x'}_0-t=\frac{p's'}{r'q'}-t=\frac{p(s-qt^2)}{q(r-pt)}-t=\frac{ps-rtp}{q(r-pt)},
\end{equation}
and is equal to 
$\Phi(x_{0})=sign(r')\sqrt{\frac{{r'}^2}{s'}+\frac{{p'}^2}{q'}} = sign(r-pt)\sqrt{\frac{{(r-pt)}^2}{s-qt^2}+\frac{p^2}{q}}.$  
This point is minimal if $r-pt<0$ and maximal if $r-pt>0$. If $r'=r-pt=0$ then the function $\Phi(x)$ is monotone increasing if $p'=p>0$ and monotone decreasing if $p'=p<0$. Now we have six cases:

\begin{description}

\item[Case 1:] If $r'=r-pt=0$ and $p>0$ then the function is monotone increasing, and therefore, $\Phi^{*}=\Phi(\infty)=p/\sqrt{q}$.

\item[Case 2:] If $r'=r-pt=0$ and $p<0$ then the function is monotone decreasing, and therefore, $\Phi^{*} = \Phi(0) = r/\sqrt{s}$.

\item[Case 3:] If $r'=r-pt>0$ and $ps-rtq\geq 0$ then by~(\ref{x0}) we get $x_0= (ps-rtq)/q(r-pt)> 0$ and is a maximal point. Therefore, $\Phi^{*} = \Phi(x_0)$.

\item[Case 4:] If $r'=r-pt>0$ and $ps-rtq < 0$ then by~(\ref{x0}) we get $x_0<0$ and is a maximal point. Therefore, the function is monotone decreasing for $x>0$; hence, $\Phi^{*} = \Phi(x_0)$.

\item[Case 5:] If $r'=r-pt<0$ and $ps-rtq\geq 0$ then by~(\ref{x0}) we get $x_0\leq 0$ and is a minimal point. Therefore, $\Phi(x)$ is monotone increasing for $x>0$ and $\Phi^{*} = \Phi(\infty)=p/\sqrt{q}$.

\item[Case 6:]If $r'=r-pt<0$ and $ps-rtq < 0$ then by (\ref{x0}) we get $x_0> 0$ and is a minimal point. Therefore, $\Phi^{*} = max(\Phi(\infty),\Phi(0))=max\left ( \frac{p}{\sqrt{q}},\frac{r}{\sqrt{s}}\right )$. 
\end{description}

\qed\end{proof}

\end{document}